%% file: main.tex
\documentclass[letterpaper]{article} 
\usepackage{aaai2026}  
\usepackage{times}  
\usepackage{helvet}  
\usepackage{courier}  
\usepackage[hyphens]{url}  
\usepackage{graphicx} 
\usepackage{natbib}  
\usepackage{caption} 
\input{preamble}

\usepackage{newfloat}
\usepackage{listings}
\DeclareCaptionStyle{ruled}{labelfont=normalfont,labelsep=colon,strut=off} 
\title{Toward the Frontiers of Reliable Diffusion Sampling \\ via Adversarial Sinkhorn Attention Guidance}
\author{Kwanyoung Kim}
\affiliations{
    Samsung Research\\
    k$\_$0.kim@samsung.com
}

\begin{document}

\twocolumn[{%
	\renewcommand\twocolumn[1][]{#1}%
	\maketitle
        \vspace{-1em}
	\begin{center}
		\centering
		\captionsetup{type=figure}
		\includegraphics[width=0.95\linewidth]{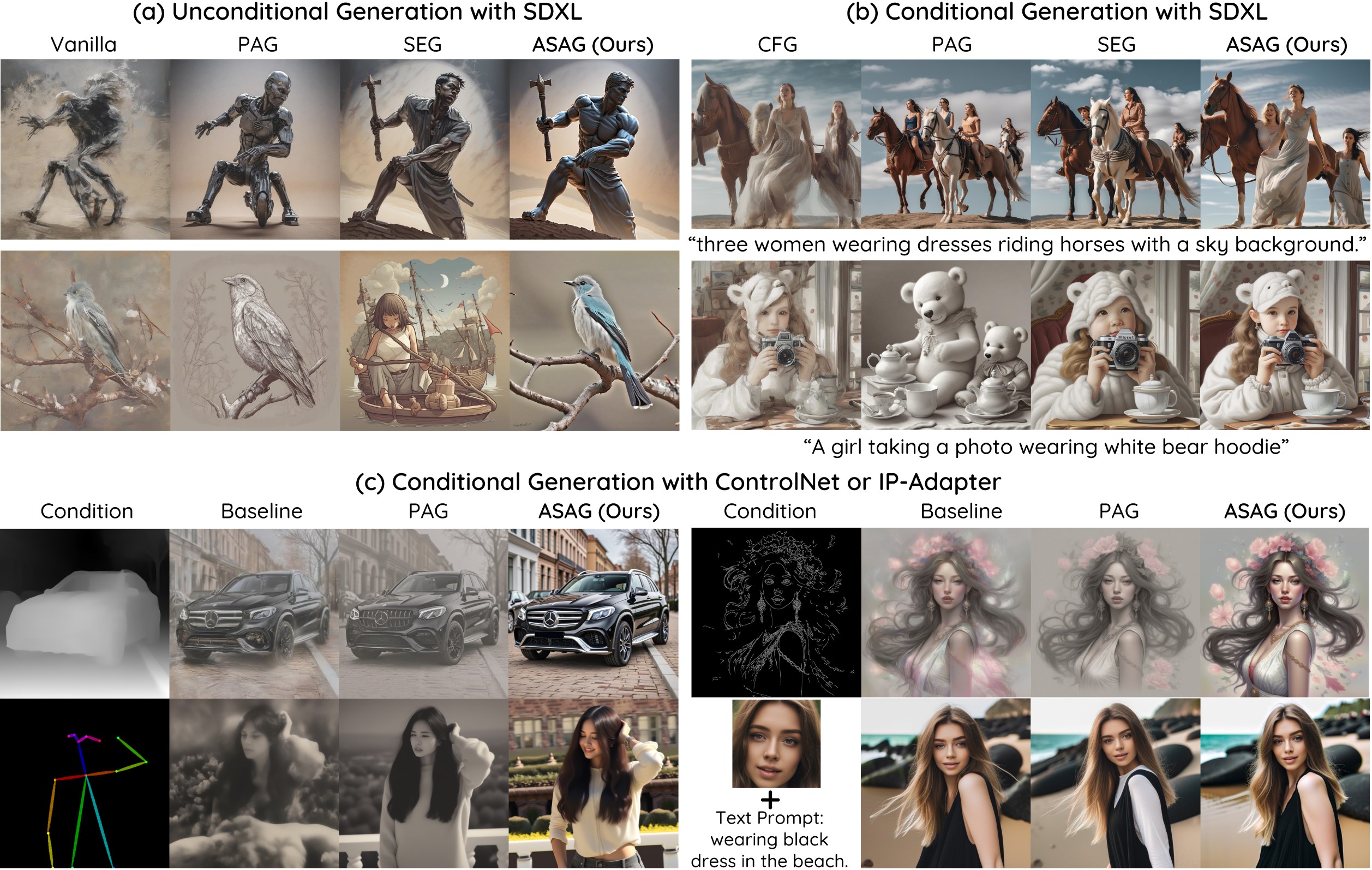} 
            \vspace{-0.5em}
		\captionof{figure}{Qualitative comparison. (a) unconditional generation, (b) conditional generation with other guidance sampling methods, and (c) conditional generation using ControlNet and IP-Adapter. Our method, \textbf{ASAG}, significantly improves visual quality in both unconditional and conditional settings. It also remarkably enhances external frameworks like ControlNet and IP-Adapter. Crucially, ASAG requires no additional training, making it broadly compatible and readily deployable.}\label{fig:main}
	\end{center}}
]

\begin{abstract}
Diffusion models have demonstrated strong generative performance when using guidance methods such as classifier-free guidance (CFG), which enhance output quality by modifying the sampling trajectory. These methods typically improve a target output by intentionally degrading another, often the unconditional output, using heuristic perturbation functions such as identity mixing or blurred conditions. However, these approaches lack a principled foundation and rely on manually designed distortions.
In this work, we propose \textbf{A}dversarial \textbf{S}inkhorn \textbf{A}ttention \textbf{G}uidance (\textbf{ASAG}), a novel method that reinterprets attention scores in diffusion models through the lens of optimal transport and intentionally disrupt the transport cost via Sinkhorn algorithm. Instead of naively corrupting the attention mechanism, ASAG injects an adversarial cost within self-attention layers to reduce pixel-wise similarity between queries and keys. This deliberate degradation weakens misleading attention alignments and leads to improved conditional and unconditional sample quality.
ASAG shows consistent improvements in text-to-image diffusion, and enhances controllability and fidelity in downstream applications such as IP-Adapter and ControlNet. The method is lightweight, plug-and-play, and improves reliability without requiring any model retraining.
\end{abstract}


\section{Introduction}
Diffusion models have led to substantial progress in the field of image and video generation~\cite{stablediffusion,stablediffusion3,dreambooth,pixart2,sana,stablevideo,videocrafter2}. Despite their effectiveness, direct or naïve sampling often results in subpar output quality. A widely adopted solution is Classifier-Free Guidance (CFG)~\cite{CFG}, which enhances class-conditional generation by computing the difference between the score functions of conditional and unconditional models and applying a weighted adjustment. Although CFG improves sample fidelity, it introduces additional training and can lead to degraded outputs when the guidance scale is too large.

Inspired by CFG, a number of guidance-based sampling techniques have emerged~\cite{SAG,AutoGuide,PAG,SEG,CFG++,TSG,SelfG}. A common strategy involves generating "weakened outputs" to serve as auxiliary signals that guide the primary model toward better sampling. While such approaches have shown empirical success, they come with inherent limitations. For example, AutoGuidance (AG) \cite{AutoGuide} relies on a poorly trained model, which is often unstable and difficult to optimize. To avoid retraining, attention based alternatives have been proposed. Perturbed Attention Guidance~\cite{PAG} distorts attention maps using identity masking, and Smooth Energy Guidance (SEG)~\cite{SEG} applies Gaussian blur to attention weights. 

Although these methods aim to weaken model outputs for improved guidance, they rely on naive heuristic functions and lack clear theoretical justification as to why such perturbations consistently enhance sample quality. This theoretical gap naturally raises a fundamental question: \textit{what constitutes an optimal perturbation for weakening outputs in a theoretically grounded manner?}

To address this question, we first revisit the intrinsic connection between attention mechanisms and optimal transport (OT) theory. One of the key contributions of this paper is the discovery and reinterpretation of classical results from optimal transport, highlighting how attention mechanisms can be framed within the OT framework. Prior studies, such as those by~\cite{sinkformers,otseg}, have shown that attention computations can be improved by employing OT-inspired methods, specifically the Sinkhorn-Knopp algorithm~\cite{sinkhorn}, to produce doubly stochastic attention maps. While these previous studies have largely utilized OT theory to boost attention performance in vision and language tasks~\cite{chen2020graph,liu2020semantic,chen2022prompt}, our approach deviates significantly from this established perspective.

Specifically, we propose a novel guidance framework, termed Adversarial Sinkhorn Attention Guidance (ASAG), which explicitly reinterprets self-attention scores—representing pixel-level similarities and interactions—through the lens of OT theory. In contrast to classical OT-based approaches that minimize transport cost to encourage semantic alignment between image embeddings, ASAG adopts an adversarial perspective by intentionally minimizing their interactions. This leads to an entropy-maximizing attention map, which corresponds to an optimally perturbed self-attention distribution.

In doing so, we establish a theoretically grounded method for systematically disrupting similarity-based alignments in diffusion models. To the best of our knowledge, this is the first work to leverage OT theory to construct an adversarial attention perturbation, leading to significant improvements in the fidelity and controllability of diffusion-based image generation.

Building upon these theoretical foundations and insights, we demonstrate that our proposed method, ASAG, effectively improves generation quality in both unconditional and conditional diffusion sampling settings. Furthermore, extensive experiments show that when combined with existing frameworks such as ControlNet and IP-Adapter, ASAG consistently outperforms other guidance approaches by a significant margin. Our key contributions can be summarized:

\begin{itemize} \item We propose ASAG, a novel and theoretically grounded diffusion guidance method that adversarially disrupts attention mechanisms by intentionally minimizing interaction between image embeddings.

\item We provide an in-depth theoretical analysis, leveraging OT theory to establish a clear and rigorous foundation for our guidance method. To the best of our knowledge, our study is the first to reinterpret OT theory from an adversarial perspective for perturbing attention scores to improve diffusion generative performance.

\item We empirically demonstrate that ASAG significantly enhances performance across general generative tasks, including unconditional and conditional image generation. Furthermore, we show the broad applicability and generalizability of ASAG by integrating it with widely-used generative frameworks such as ControlNet and IP-Adapter, achieving substantial improvements over existing guidance approaches. \end{itemize}

\section{Preliminary}
\subsection{Diffusion Models}
Diffusion models (DM)~\cite{ho2020denoising, song2021scorebased} are a class of generative models that produce samples by reversing a known forward diffusion process through estimation of the score function of a given data distribution. Specifically, given samples $\mbx_0 \sim q_{\text{data}}(\mathbf{x})$, DMs define a forward noise-adding Markov process 
\(q(\mbx_{t}|\mbx_{t-1}) = \mathcal{N}(\sqrt{1-\beta_t}\mbx_{t-1}, \beta_t \mathbf{I}), \quad t=1,\dots,T,\)
where the variance schedule $\{\beta_t\}_{t=1,\dots,T}$ is predefined. 

Consequently, the distribution at any intermediate timestep $t$ can be explicitly expressed as
\(q(\mbx_t) = \mathcal{N}(\sqrt{\bar{\alpha}_t}\mbx_0, (1-\bar{\alpha}_t)\mathbf{I}),\)
with the cumulative coefficient defined as $\bar{\alpha}_t = \prod_{i=1}^{t}\alpha_i$, where $\alpha_t = 1 - \beta_t$. As $t$ approaches $T$, the distribution $q(\mbx_T)$ converges to an isotropic Gaussian distribution, \( q(\mbx_T) \approx \mathcal{N}(0,\mathbf{I}).\)
To generate samples, diffusion models parameterize the reverse diffusion process as,
\(p_\theta(\mbx_{t-1}|\mbx_t) = \mathcal{N}(\bs{\mu}_\theta(\mbx_t,t), \Sigma_\theta(\mbx_t,t)),\)
where the model parameters $\theta$ can be optimized by denoising score matching (DSM) objective~\cite{vincent2011connection}:
\begin{align}
\underset{\theta}{\min}\;\mathbb{E}_{\mbx_0,\me}\left[ \|\me_\theta(\sqrt{\bar{\alpha}_t}\mbx_0+\sqrt{1-\bar{\alpha}_t}\me, t) - \me \|^2_2\right],
\quad  \label{eq:dsm}
\end{align}
where $\me \sim \mathcal{N}(0,\mathbf{I})$.
Once trained, sampling from a diffusion model proceeds by sequentially reversing the diffusion steps starting from an isotropic Gaussian noise sample. For instance, the Denoising Diffusion Implicit Model (DDIM)~\cite{song2021ddim} generates samples by iteratively applying the update rule:
\begin{align}
\mbx_{t-1} = \sqrt{\bar{\alpha}_{t-1}}\hat{\mbx}_0(t) + \sqrt{1-\bar{\alpha}_{t-1}}\me_{\theta}(\mbx_t,t),
\end{align}
where the denoised estimate $\hat{\mbx}_0(t) = \mathbb{E}[\mbx_0|\mbx_t]$ is computed using Tweedie's formula~\cite{efron2011tweedie,kim2021noise2score}:
\begin{align}
\hat{\mbx}_0(t) = \frac{\mbx_t - \sqrt{1-\bar{\alpha}_t}\me_\theta(\mbx_t,t)}{\sqrt{\bar{\alpha}_t}}.
\end{align}
The sampling step described above is repeated recursively from timestep $T$ down to timestep $1$.

\subsection{Guidance Sampling in Diffusion Models}\label{sec:pre_guidance}
To enhance generation with arbitrary conditions (typically class or textual embeddings), various guidance-based sampling methods have been proposed~\cite{CG,CFG,CFG++,SAG,AutoGuide,PAG,SEG,TSG}. 
Let the conditional model be defined as $\me_{\theta}(\mbx_t,t,\mbc)$, which we simplify as $\me_{\theta}(\mbx_t,\mbc)$, and the unconditional counterpart as $\me_{\theta}(\mbx_t,\mathbf{\varnothing})$.

Several studies~\cite{CFG,SAG,PAG} have proposed generalized guidance frameworks using imaginary labels. We revisit this idea with an implicit discriminator $\mc{D}(\mbx_t) = \frac{p(\mbf{y}|\mbx_t)}{p(\tilde{\mbf{y}}|\mbx_t)}$, which distinguishes desirable from undesirable samples during the diffusion process. Here, $\mbf{y}$ and $\tilde{\mbf{y}}$ denote imaginary desirable and undesirable labels, respectively.

Similar to CFG, which uses an implicit classifier, implicit discriminator $\mc{D}$ encourages sampling along desirable trajectories while suppressing undesirable ones. By applying Bayes’ rule and a mathematical transformation, we derive the following guided sampling objective:
\begin{align}
\me'_{\theta}(\mbx_t,\mbc) 
&= \me_{\theta}(\mbx_t,\mbc) - s \sigma_t \nabla_{\mbx_t} \left( \log p(\mbx_t|\mbf{y}) - \log p(\mbx_t|\tilde{\mbf{y}}) \right) \notag \\
&= \me_{\theta}(\mbx_t,\mbc) + s \left( \me_{\theta}(\mbx_t,\mbc) - \tilde{\me}_{\theta}(\mbx_t,\mbc) \right), \label{eq:weak}
\end{align}
where $s$ is the guidance strength. Since the diffusion model learns to approximate the score function of the desirable distribution, the term $\sigma_t \nabla_{\mbx_t} \log p(\mbx_t|\mbf{y})$ can be replaced by $\me_{\theta}(\mbx_t,\mbc)$. The term $\tilde{\me}_{\theta}$ is a heuristically constructed weaker variant that simulates an undesirable score. For example, in CFG, the class condition is dropped, while in PAG, an identity condition is injected to produce undesirable outputs.

Although these approaches have shown empirical success and clearly demonstrate the necessity of modeling undesirable paths, the theoretical justification for how to construct such paths remains limited.


\subsection{Optimal Transport and Sinkhorn-Knopp}
\label{sec:sinkhorn}
In classical discrete optimal transport (OT), given a predefined cost matrix $\mathbf{M} \in \mathbb{R}^{n\times n}$ and two probability vectors $\boldsymbol{\bs{\mu}}, \boldsymbol{\nu} \in \Sigma_n =\{\boldsymbol{p} \in \mathbb{R}^{n}: \boldsymbol{p} \geq 0, \mathbf{1}^{\top}\boldsymbol{p}=1\}$, the OT problem is formulated as:
\begin{align}
d_{\mathbf{M}}(\bs{\bs{\mu}},\bs{\nu}) = \underset{\bs{P}\in \mc{U}(\bs{\bs{\mu}},\bs{\nu})}{\min}\langle\mbf{P},\mbf{M}\rangle,
\end{align}
where $\mc{U}(\bs{\bs{\mu}},\bs{\nu}) = \{\mbf{P} \in \mathbb{R}_{+}^{n \times n} |\mbf{P}\mbf{1}_{n} = \bs{\bs{\mu}},\mbf{P}^{\top}\mbf{1}_{n} = \bs{\nu}\}$ denotes the transport polytope consisting of non-negative matrices with prescribed row and column sums, and $\langle\cdot,\cdot\rangle$ indicates the Frobenius inner product. Directly solving the OT problem incurs a computational cost of $O(n^3\log n)$, which is often prohibitively expensive.

To overcome this, the entropy-regularized OT formulation, solved via the Sinkhorn-Knopp (Sinkhorn) algorithm~\cite{sinkhorn}, is defined as:
\begin{align}
d^{\lambda}_{\mathbf{M}}(\bs{\bs{\mu}},\bs{\nu}) = \underset{\bs{P}\in \mc{U}(\bs{\bs{\mu}},\bs{\nu})}{\min} \langle\mbf{P},\mbf{M}\rangle -\frac{1}{\lambda}\langle\mbf{P},\log \mbf{P}\rangle,
\label{Sinkhorn}
\end{align}
where $\lambda >0$ is the regularization parameter, and the second term is the entropy regularizer. This problem is efficiently solved by iteratively updating scaling vectors. Specifically, at iteration $i \rightarrow \infty$, the optimal transport plan via Sinkhorn, $\texttt{Sinkhorn}(\lambda\mbf{M}):= \mbf{P}^{\ast}$, can be expressed as:
\begin{align}
\mbf{P}^{\ast} = \text{diag}(\bs{u}^{i})\exp(-\lambda \mbf{M})\text{diag}(\bs{v}^{i}),
\label{Sinkhorn2}
\end{align}
where scaling vectors $\bs{u}^{i}$ and $\bs{v}^{i}$ are updated via:
\begin{align}
\bs{u}^{i} = \frac{\boldsymbol{\bs{\mu}}}{\exp(-\lambda \mbf{M})\bs{v}^{i-1}},\quad
\bs{v}^{i} = \frac{\boldsymbol{\nu}}{\exp(-\lambda \mbf{M})\bs{u}^{i}},
\end{align}
with initialization $\bs{v}^{0}=\bs{1}$. To improve numerical stability, we adopt the log-domain scaling version of Sinkhorn optimization~\cite{schmitzer2019stabilized}.

\subsection{Connecting Self-Attention with Sinkhorn}
Given an input sequence $\mathbf{X} = [x_1, x_2, \dots, x_n]$ embedded in a $d$-dimensional space, the self-attention mechanism in Transformer models is defined as:
\begin{align}
\text{SA}(\mbf{Q},\mbf{K},\mbf{V}) &= \texttt{SoftMax}\left(\frac{\mbf{Q}\mbf{K}^{\top}}{\sqrt{d}}\right)\mbf{V}, \label{self-atten}\\
\text{where}\quad \mathbf{Q} = \phi_q(\mathbf{X}), &\;\mathbf{K} = \phi_k(\mathbf{X}), \;\mathbf{V} = \phi_v(\mathbf{X}).\notag
\end{align}
Here, $\mbf{Q}, \mbf{K}, \mbf{V} \in \mathbb{R}^{n \times d}$ denote the query, key, and value matrices obtained via learned linear projections, each corresponding $\phi_q(\cdot)$, $\phi_k(\cdot)$, and $\phi_v(\cdot)$

It has been shown that the first iteration of the Sinkhorn algorithm corresponds exactly to the \texttt{SoftMax}~\cite{sinkformers}, suggesting a natural generalization of attention via optimal transport. Specifically, by interpreting the attention score as a similarity matrix and defining the cost as $\mbf{M} = (\mathbf{1} - \mbf{Q}\mbf{K}^{\top})$, Sinkhorn-based attention (Sink-Attention) optimizes a doubly-stochastic plan that favors high similarity alignments. This leads to more structured and expressive attention maps compared to standard softmax-based attention, and empirically improves performance~\cite{otseg}. The Sink-Attention can be formulated as:
\begin{align}
\text{Sink-A}(\mbf{Q},\mbf{K},\mbf{V}) &= \texttt{Sinkhorn}\left(\lambda\mbf{M}\right)\mbf{V},\label{sink-a}
\end{align}
where \texttt{Sinkhorn}($\cdot$) computes a doubly-stochastic matrix via Eq.\eqref{Sinkhorn2}, with regularization parameter $\lambda$ set to $1/\sqrt{d}$.

\begin{figure*}[ht!]
\centering
\includegraphics[width=0.8\linewidth]{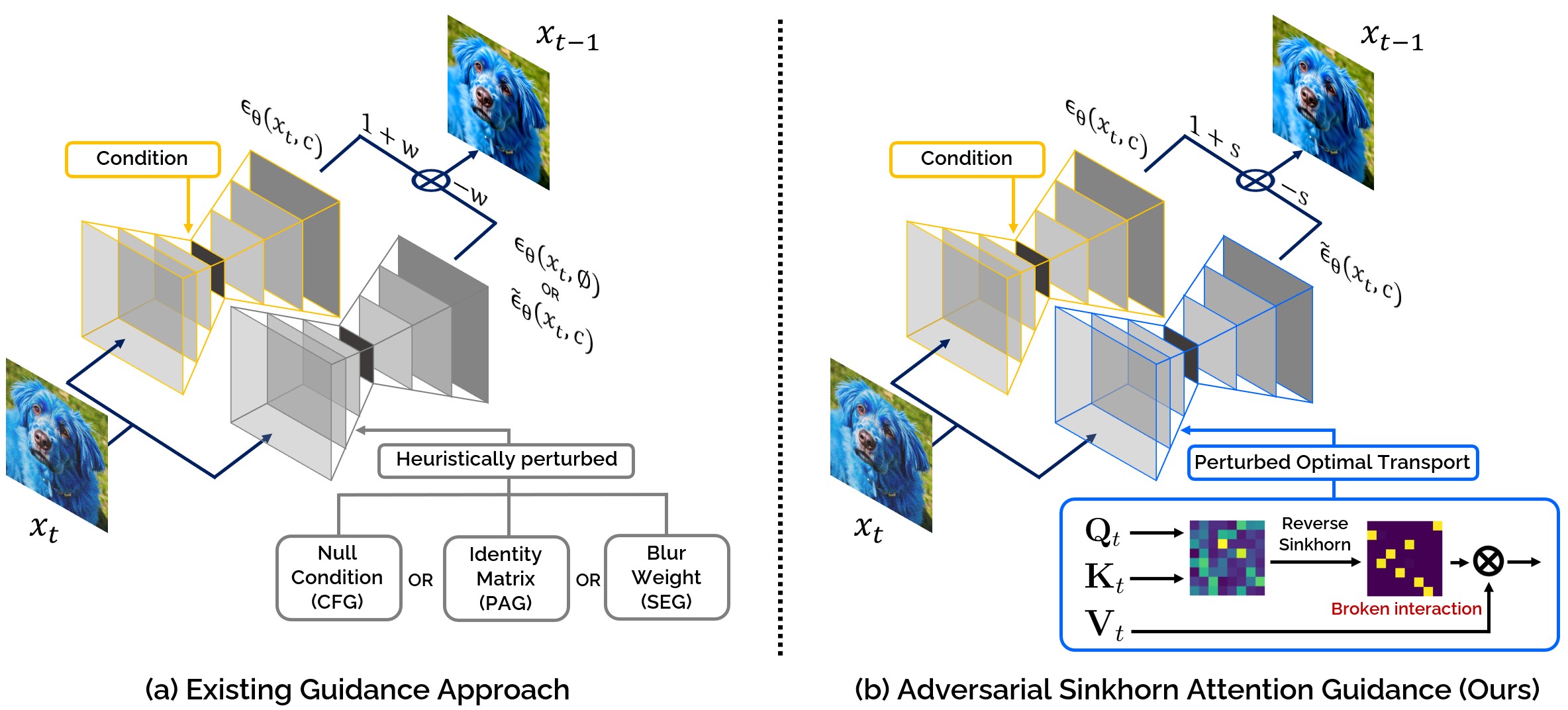}
\vspace{-0.5em}
\caption{Conceptual comparison between ASAG and other guidance methods.
Existing guidance methods often rely on null conditions or heuristic perturbations of self-attention, such as injecting identity matrices or applying Gaussian blurs, to simulate undesirable paths.
In contrast, ASAG explicitly defines an attention cost function based on pixel-level interactions and disrupts attention scores by minimizing this cost, thereby intentionally breaking semantic interactions through the Sinkhorn algorithm.}
\label{fig:concept}
\vspace{-1em}
\end{figure*}

\section{Methods}

\subsection{Self- and Sinkhorn Attention in Diffusion Models}
Recent DMs extensively utilize attention mechanisms throughout their architecture. At each timestep $t$, the model processes features via self-attention, typically expressed as:
\begin{align}
\text{SA}(\mbf{Q}_t,\mbf{K}_t,\mbf{V}_t) &= \texttt{SoftMax}\left(\frac{\mbf{Q}_t\mbf{K}_t^{\top}}{\sqrt{d}}\right)\mbf{V}_t. \label{self-atten-DM}
\end{align}
As introduced in Eq.~\ref{sink-a}, an alternative based on optimal transport is Sinkhorn attention, which replaces the \texttt{Softmax} with \texttt{Sinkhorn} operator which compute a transport plan via iterative optimization:
\begin{align}
\text{Sink-A}(\mbf{Q}_t,\mbf{K}_t,\mbf{V}_t) &= \texttt{Sinkhorn}(\lambda\mbf{M}^{\uparrow}_t)\mbf{V}_t, \label{sink-atten_dm}
\end{align}
where $\lambda = 1/\sqrt{d}$, $\mbf{M}^{\uparrow}_t = (\mathbf{1} - \mbf{Q}_t\mbf{K}_t^{\top})$. Minimizing this cost means that maximize similarity between query and key matrices at $t$. While Sink-A can improve alignment, it is computationally expensive due to its iterative nature. When applied to all attention layers, the overhead becomes even more significant.
However, our goal is not to enhance all attention mechanisms. Instead, within the guidance sampling framework, we aim to construct an effective undesirable path $\tilde{\me}_{\theta}(\mbx_t,\mbc)$ that simulates degraded attention behavior.

\subsection{Adversarial Sinkhorn Attention Guidance}

To construct the undesirable score function $\tilde{\me}_{\theta}(\mbx_t,\mbc)$, we propose Adversarial Sinkhorn Attention Guidance (ASAG), which selectively applies Sink-A in reverse direction to degrade attention quality in specific layers as shown Fig~\ref{fig:concept}. This design is inspired by prior works such as PAG and SEG, which perturb only a subset of attention maps.

In ASAG, we replace cost $\mbf{M}^{\uparrow}_t$ in Eq.~\eqref{sink-atten_dm} with $\mbf{M}_t^{\downarrow} = (\mbf{Q}_t\mbf{K}_t^{\top})$  for direction minimizing similarity between query and key matrices:
\begin{align}
\text{ASA}(\mbf{Q}_t,\mbf{K}_t,\mbf{V}_t) &= \texttt{Sinkhorn}\left(\lambda\mbf{M}_t^{\downarrow}\right)\mbf{V}_t, \label{sink-atten_dm2}
\end{align}
where ASA denotes Adversarial Sinkhorn Attention. Unlike Sink-A with the OT objective that maximize similarity, ASA seeks to minimize interaction between noisy queries and keys, leading to attention collapse and disrupted interactions as detailed Algorithm~\ref{algo-ASA}.

Since degradation does not require precise convergence, only a few Sinkhorn iterations are sufficient, resulting in minimal computational overhead. The resulting attention maps simulate an undesirable path where semantic coherence is lost. Using ASA, we compute $\tilde{\me}_{\theta}(\mbx_t,\mbc)$ in Eq.~\eqref{eq:weak},  and pseudo code are provided in Algorithm~\ref{algo-ASAG}.

\begin{algorithm}[t!]
\caption{Adversarial Sinkhorn Attention}
\label{algo-ASA}
\SetKwInOut{Input}{Input}
\SetKwInOut{Output}{Output}
\SetKwInput{kwInit}{Initialization}

\Input{
    query $\mbf{Q}_t$ , key $\mbf{K}_t$, and value $\mbf{V}_t$ matrices at timestep $t$,
    hyper-parameter $\lambda = 1/\sqrt{d}$, 
    error threshold $\epsilon_{\text{max}}$
}
\kwInit{Attention cost $\mbf{M}_t^{\downarrow} = (\mbf{Q}_t \mbf{K}_t^{\top})$,
    $\Delta_{\bs{v}} = \infty$, $i = 1$, $\bs{v}^0 = 0$
}
\text{Calculate Sinkhorn distance within inner loop} \;
\While{$\Delta_{\bs{v}} < \epsilon_{\text{max}}$}{
    $\bs{u}^i = \log{\bs{\bs{\mu}}} - \log \left[\sum \exp \left(-\lambda \mbf{M}_{t}^{\downarrow} + \bs{v}^{i-1} \right) \right]$ \; 
    $\bs{v}^i = \log{\bs{\nu}} - \log \left[\sum \exp \left(-\lambda (\mbf{M}_{t}^{\downarrow})^{\top} + \bs{u}^{i} \right) \right]$ \;
    $\Delta_{\bs{v}} = \|\bs{v}^i - \bs{v}^{i-1} \|_{1}$ \;
    $i \leftarrow i + 1$ \;
}
\text{$\mbf{P^*}$ =
     $\text{diag}(\exp(\bs{u}^i)) \cdot \exp(-\lambda \mbf{M}_{t}^{\downarrow}) \cdot \text{diag}(\exp(\bs{v}^i))$} \\
\Output{$\text{ASA}(\mbf{Q}_t,\mbf{K}_t,\mbf{V}_t)=\mbf{P}^{*} \mbf{V}_t$}
\end{algorithm}

\begin{algorithm}[t!]
	\caption{Diffusion Sampling with ASAG}
	\label{algo-ASAG}
	\SetKwInOut{Input}{Input}
	\SetKwInput{kwInit}{return}
	\Input{Diffusion model $\me_{\theta}(\mbx_t)$ with self-attention module, $\tilde{\me}_{\theta}(\mbx_t,\mbc)$ with ASA \\
    guidance scale $s$.}
\For{$t$ in $T,T-1, \cdots, 1$ }{
$\me_{\theta}(\mbx_t,\mbc) \leftarrow \me_{\theta}(\mbx_t,\mbc) + s \left( \me_{\theta}(\mbx_t,\mbc) - \tilde{\me}_{\theta}(\mbx_t,\mbc) \right)$ \\
$\hat \mbx_0(t) \leftarrow (\mbx_t - \sqrt{1- \bar{\alpha_t}}\me_{\theta}(\mbx_t,\mbc)) / \sqrt{\bar{\alpha_t}}$ \\
$\mbx_{t-1} \leftarrow \sqrt{\bar{\alpha}_{t-1}} \hat \mbx_0(t)  + \sqrt{1- \bar{\alpha}_{t-1}}\me_{\theta}(\mbx_t,\mbc)$}
\kwInit{$\mbx_{0}$}
\end{algorithm}





\paragraph{Theoretical Justification for ASAG.}
We formally justify ASA’s perturbation as an entropy-maximizing transport plan that disrupts semantic alignment. This direction is not arbitrary but results from a constrained optimization aligned with a disruptive axis in score space.

\begin{thm}[Entropy-Maximizing Plan via Adversarial Sinkhorn]~\label{thm:1}
Let $\mbf{Q}_t, \mbf{K}_t \in \mathbb{R}^{n \times d}$ be the query and key matrices at diffusion timestep $t$. Define the adversarial cost matrix as $\mbf{M}_t^{\downarrow} = (\mbf{Q}_t\mbf{K}_t^{\top})$. The entropy-regularized OT problem is defined as 
\(d^{\lambda}_{\mbf{M}_t^{\downarrow}}(\bs{\bs{\mu}},\bs{\nu}) = \underset{\bs{P}\in \mc{U}(\bs{\bs{\mu}},\bs{\nu})}{\min} \langle\mbf{P},\mbf{M}_t^{\downarrow}\rangle -\frac{1}{\lambda}\langle\mbf{P},\log \mbf{P}\rangle.
\)
Then in the limit $\lambda \to 0$ (i.e., $1/\lambda \to \infty$), the solution converges to the maximum-entropy plan:
\[
\lim_{\lambda \to 0} \mbf{P}_t^* = \frac{1}{n^2} \mathbf{1} \mathbf{1}^\top.
\]
\end{thm}

\begin{lemma}[Uniform Plan Maximizes Entropy]\label{lemma:uniform_entropy}
Under uniform marginals \(\bs{\mu} \) and \(\bs{\nu}\), the coupling $\mbf{P}^* = \frac{1}{n^2}\,\mathbf{1}\,\mathbf{1}^\top$
uniquely maximizes the Shannon entropy over the transport polytope \(\mathcal U(\bs{\mu},\bs{\nu})\).
\end{lemma}
All proofs are deferred in supplement.
\begin{remark}
By Theorem~\ref{thm:1} and Lemma~\ref{lemma:uniform_entropy}, the adversarial Sinkhorn plan converges to the maximum-entropy uniform coupling as \(\lambda \to 0\), effectively diminishing semantic preferences. This leads to an increasingly unstructured attention map, representing a limiting case of semantic alignment degradation. ASA leverages this behavior to construct an adversarial attention map that systematically disrupts semantic correspondence.
\end{remark}

\begin{corollary} Let $\delta_t := \me_{\theta}(\mbx_t,\mbc) - \tilde{\me}_{\theta}(\mbx_t,\mbc)$ be the guidance energy between the original and ASA-induced score estimates in Eq.~\ref{eq:weak}. Then  $\delta_t$ defines a semantically grounded direction in score space—anchored to the original attention and diverging from the entropy-maximizing adversarial trajectory. Unlike heuristic perturbations, it enables principled contrastive guidance by preserving structural semantics while deliberately reducing alignment.
\end{corollary}

\paragraph{Practical Justification via Sinkhorn Approximation.}
While the uniform plan $\frac{1}{n^2}\mathbf{1} \mathbf{1}^\top$ represents the theoretical extreme of semantic disruption, applying it directly often leads to reduced generation diversity and unstable behavior. ASA instead adopts a Sinkhorn plan with a small but finite $\lambda$, which retains a doubly stochastic structure and enables controlled entropy increase while preserving attention stability. Unlike heuristic perturbations, this approach maintains the optimization structure of attention with a reversed objective, offering a principled and tunable guidance strategy grounded in theory and robust in practice.

\begin{figure*}[t!]
\centering
\includegraphics[width=0.9\linewidth]{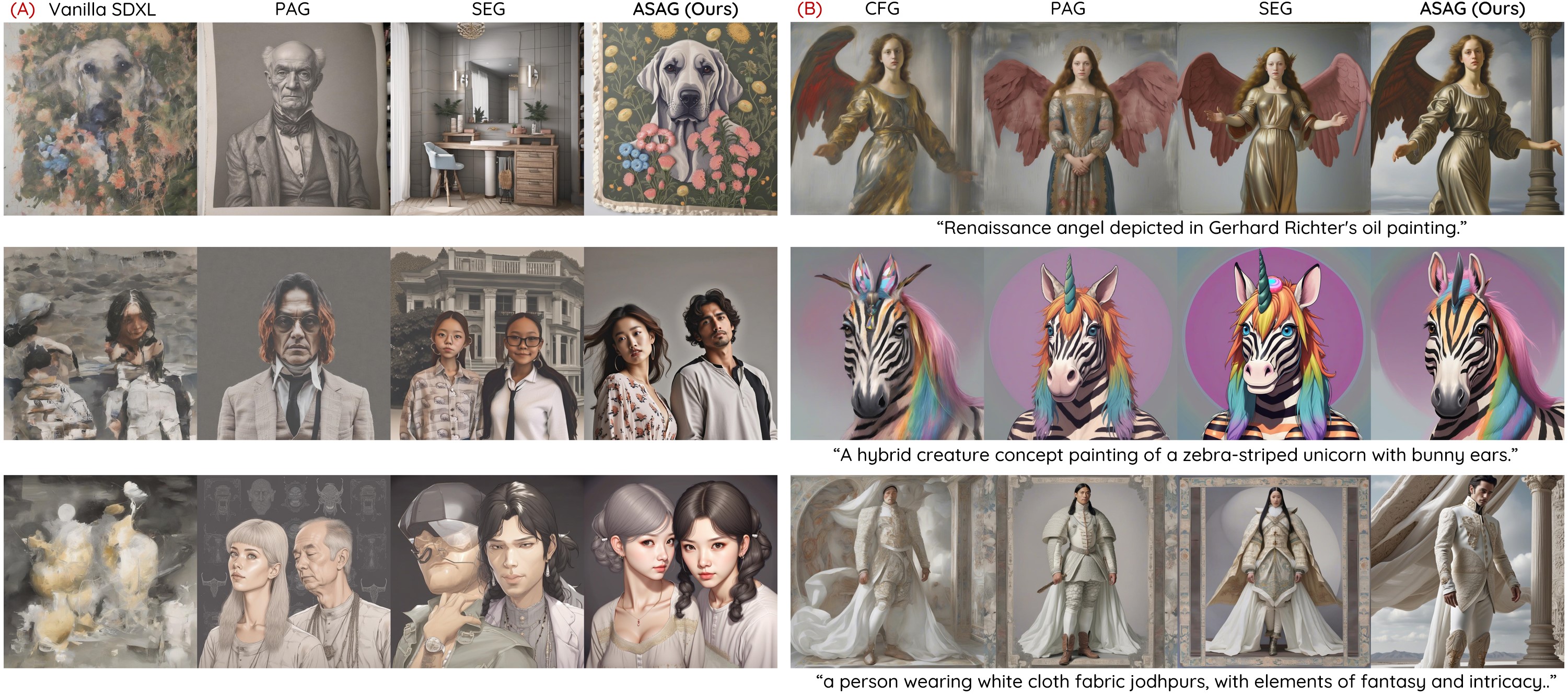}
\vspace{-0.5em}
\caption{Comparison results on (A) unconditional and (B) conditional generation using Vanilla, CFG, PAG, SEG, and ASAG. While other guidance methods often alter the structure of the original outputs, ASAG achieves both higher visual quality and stronger consistency in structure and intent.}
\label{fig:uncond}
\vspace{-0.5em}
\end{figure*}

\begin{figure*}[ht!]
\centering
\includegraphics[width=0.9\linewidth]{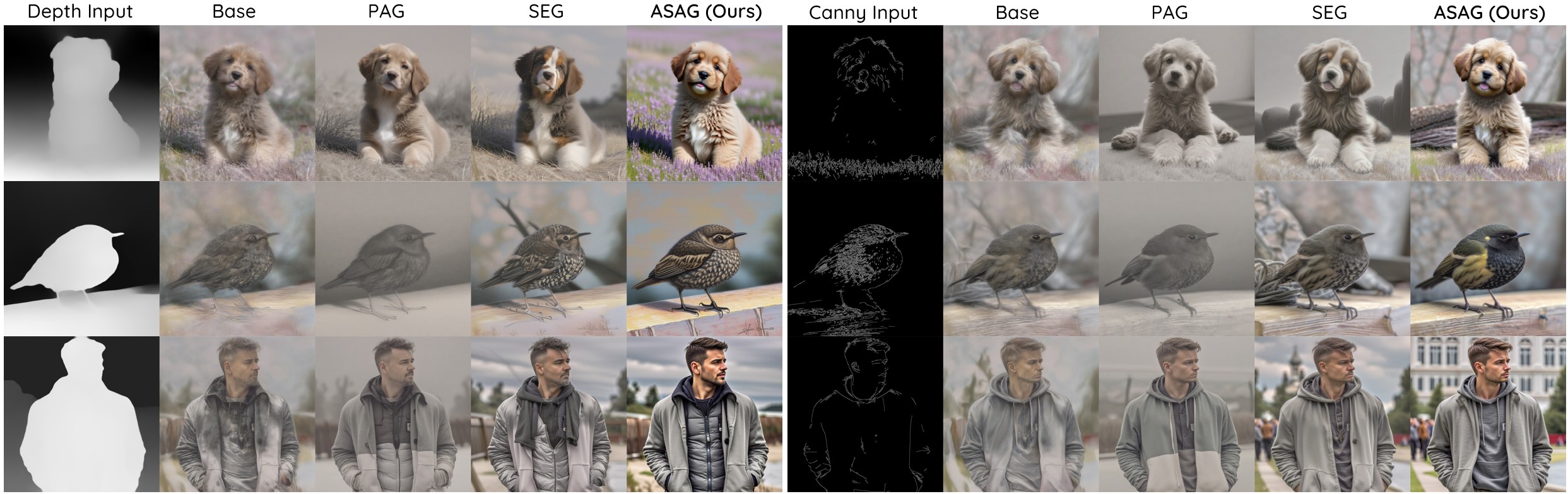}
\vspace{-0.5em}
\caption{ControlNet examples with different guidance sampling methods.
Left: Canny condition; Right: Depth condition.
Our method, when integrated with ControlNet, substantially improves visual quality and preserves fine-grained image details.}
\label{fig:control_1}
\vspace{-0.5em}
\end{figure*}

\begin{figure}[ht!]
\centering
\includegraphics[width=0.95\linewidth]{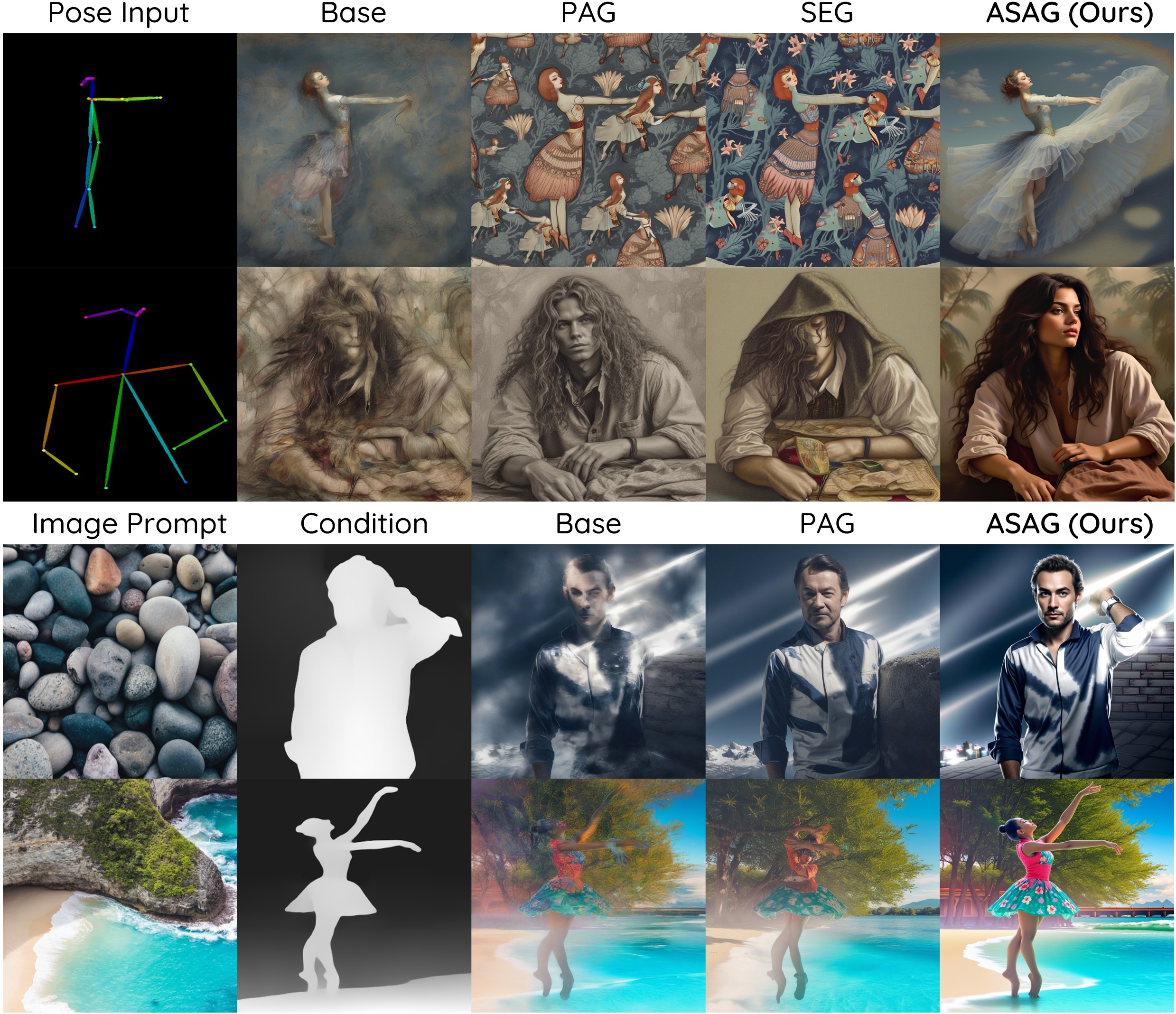}
\vspace{-0.5em}
\caption{Comparison of guidance sampling methods combined with ControlNet and IP-Adapter under pose and depth conditions. Our method significantly enhances image quality, yielding clearer structures.}
\label{fig:control_2}
\vspace{-1em}
\end{figure}

\section{Experiments}
\paragraph{Setup.} For fair comparison, we use SDXL\cite{sdxl} and SD3~\cite{stable3}. We compare against PAG and SEG with a guidance scale of 3.0, following their official settings. SEG is not officially supported on SD3 and IP-Adapter, so it is omitted in those cases. For ASAG, we set the guidance scale $s = 1.5$ and use 25 sampling steps. Full implementation details are provided in the supplement. All experiments are run on a single NVIDIA H100 GPU.

\paragraph{Evaluation Metrics.} We evaluate our method across multiple dimensions. For visual quality, we report the Fréchet Inception Distance (FID)~\cite{fid} and Kernel Inception Distance (KID), and Inception score (Diversity) using 30K randomly sampled prompts from the MS-COCO validation set~\cite{coco}. To assess text-image alignment and human preference, we use CLIPScore~\cite{clipscore}, ImageReward~\cite{imagereward}, PickScore~\cite{pick}, and Human Preference Score (HPS v2.1)~\cite{hpsv2}. Evaluations are conducted on prompts from both MS-COCO~\cite{coco} and additional benchmarks including DrawBench~\cite{drawbench} and HPD~\cite{hpsv2}. Further evaluation details are provided in the supplement.

\begin{table}[t!]
\caption{Quantitative results of various guidance methods on the MS-COCO dataset with SDXL and SD3 in uncoditional and conditional generation. Bold text indicates the best performance for each metric. }
\vspace{-0.5em}
\label{tab_main}
\resizebox{1.0\linewidth}{!}{
\begin{small}
\begin{tabular}{ccccc}
\toprule
 Condition & Method & \multicolumn{1}{c}{FID~$\downarrow$} & \multicolumn{1}{c}{KID~$\downarrow$} & \multicolumn{1}{c}{Inception Score~$\uparrow$} \\
\cmidrule(lr){1-1}\cmidrule(lr){2-2}\cmidrule(lr){3-5}
\multirow{4}{*}{\shortstack{Unconditional \\ Generation}} &Vanilla & 122.07 & 0.086 & 7.052  \\
\cmidrule(lr){2-2}\cmidrule(lr){3-5}
& PAG & 108.63 & 0.067 & 10.46 \\
\cmidrule(lr){2-2}\cmidrule(lr){3-5}
& SEG & 95.43 & 0.062 & 10.35 \\
\cmidrule(lr){2-2}\cmidrule(lr){3-5}
\rowcolor{green!10}  \cellcolor{white} 
&\textbf{ASAG (Ours)} & \bf 92.01 & \bf 0.059  & \bf 10.54  \\
\midrule
 Condition & Method & \multicolumn{1}{c}{FID~$\downarrow$} & \multicolumn{1}{c}{CLIPScore~$\uparrow$} & \multicolumn{1}{c}{ImageReward~$\uparrow$} \\
\cmidrule(lr){1-1}\cmidrule(lr){2-2}\cmidrule(lr){3-5}
\multirow{4}{*}{\shortstack{Conditional \\ Generation (SDXL)}} & CFG & 28.15 & 25.21 & 0.415 \\
\cmidrule(lr){2-2}\cmidrule(lr){3-5}
& PAG & 24.32 & 25.41 & 0.448 \\
\cmidrule(lr){2-2}\cmidrule(lr){3-5}
& SEG & 26.80 & 25.39 & 0.431 \\
\cmidrule(lr){2-2}\cmidrule(lr){3-5}
\rowcolor{green!10}  \cellcolor{white} &
\textbf{ASAG (Ours)} & \bf 23.30 & \bf 25.85  & \bf 0.459  \\       
\midrule
 Condition & Method & \multicolumn{1}{c}{FID~$\downarrow$} & \multicolumn{1}{c}{CLIPScore~$\uparrow$} & \multicolumn{1}{c}{ImageReward~$\uparrow$} \\
\cmidrule(lr){1-1}\cmidrule(lr){2-2}\cmidrule(lr){3-5}
\multirow{4}{*}{\shortstack{Conditional \\ Generation (SD3)}} & CFG & 24.19 & 26.03 & 0.931 \\
\cmidrule(lr){2-2}\cmidrule(lr){3-5}
& PAG & 23.31 & 26.14 & 0.956 \\
\cmidrule(lr){2-2}\cmidrule(lr){3-5}
\cmidrule(lr){2-2}\cmidrule(lr){3-5}
\rowcolor{green!10}  \cellcolor{white} &
\textbf{ASAG (Ours)} & \bf 22.87 & \bf 26.33  & \bf 0.978  \\      
\bottomrule
\end{tabular}
\end{small}
}

\vspace{-0.5em}
\end{table}

\begin{table}[t!]
\caption{Quantitative comparison of text alignment and human preference across datasets using various guidance methods with SDXL. For PAG, SEG, and ASAG, CFG guidance is used jointly. }
\label{tab_hps}
\vspace{-0.5em}
\resizebox{1\linewidth}{!}{
\begin{small}
\begin{tabular}{clcccc}
\toprule
Dataset & Method & CLIPScore~$\uparrow$ & PickScore~$\uparrow$ & ImageReward~$\uparrow$ & HPSv2~$\uparrow$ \\
\cmidrule(r){1-1}\cmidrule(r){2-2}\cmidrule(r){3-6}
\multirow{4}{*}{Drawbench} & CFG & 25.61 & 21.70 & 0.196 & 26.81 \\
\cmidrule(r){2-2}\cmidrule(r){3-6}
 & PAG & 26.17 & 21.93 & 0.294 & 26.84 \\
\cmidrule(r){2-2}\cmidrule(r){3-6}
 & SEG & 26.03 & 21.78 & 0.290 & 27.06 \\
\cmidrule(r){2-2}\cmidrule(r){3-6}
\rowcolor{green!10}  \cellcolor{white} & \textbf{ASAG (Ours)}  & \bf 26.62  & \bf21.99  & \bf0.316 & \bf27.08 \\
\midrule
\multirow{4}{*}{HPD} & CFG & 27.88 & 21.97 & 0.565 & 26.63 \\
\cmidrule(r){2-2}\cmidrule(r){3-6}
\cmidrule(r){2-6}
 & PAG & 28.00 & 22.12 & 0.635 & 28.83 \\
\cmidrule(r){2-2}\cmidrule(r){3-6}
\cmidrule(r){2-6}
 & SEG & 28.15 & 21.96 & 0.622 & 28.73 \\
\cmidrule(r){2-2}\cmidrule(r){3-6}
\rowcolor{green!10}  \cellcolor{white} & \textbf{ASAG (Ours)} & \bf 28.58   & \bf 22.21  & \bf 0.673  & \bf 28.84  \\
\bottomrule	        
\end{tabular}
\end{small}
}
\vspace{-1em}
\end{table}

\subsection{Results on Diffusion Generation}

To rigorously evaluate the effectiveness of our method, we generate 30K samples on the MSCOCO dataset using various guidance sampling techniques in both unconditional and conditional generation settings.

\paragraph{Unconditional Generation}
To isolate the effect of our method, we first evaluate ASAG in an unconditional generation setup.
As shown in Tab.~\ref{tab_main}, ASAG consistently outperforms other guidance approaches across all evaluation metrics, demonstrating improved visual fidelity and greater diversity, as reflected by higher Inception Scores.
Qualitatively, ASAG also produces outputs that are more visually appealing and better aligned with the original vanilla results, particularly in cases where other guidance methods generate high-quality images that nonetheless diverge significantly from the vanilla outputs (Fig.~\ref{fig:uncond}).
These findings suggest that ASAG serves as a strong guidance strategy even in scenarios without any conditional information.

\paragraph{Conditional Generation}  
We further evaluate our method in a conditional generation setup, where guidance sampling is combined with CFG. As shown in Tab.~\ref{tab_main}, while existing guidance methods benefit from CFG, ASAG achieves superior performance in both visual quality and text-image alignment. It is also compatible with the SD3 backbone and consistently outperforms other methods, demonstrating strong generalizability. To further validate its effectiveness, we evaluate on a human preference dataset (Tab.~\ref{tab_hps}), where ASAG with CFG achieves state-of-the-art performance across all metrics. Qualitative results in Fig.~\ref{fig:uncond} show that ASAG enhances visual quality while preserving the structural intent of the original CFG output, unlike other methods that often alter the generation semantics.

\subsection{Results on Downstream Tasks}
To evaluate the effectiveness of ASAG in downstream tasks, we conduct experiments with ControlNet and IP-Adapter under various conditioning settings, keeping all configurations identical except for the guidance method.
As shown in Fig.~\ref{fig:control_1}, ASAG consistently outperforms baselines under Canny and depth conditions by better preserving structural fidelity. In more challenging setups—pose with ControlNet and multimodal IP-Adapter (Fig.~\ref{fig:control_2})—ASAG captures fine-grained details and maintains stronger visual coherence, even where other methods degrade.

Notably, these improvements are achieved without any additional training or fine-tuning, demonstrating the plug-and-play flexibility of ASAG. Its principled perturbation strategy effectively guides pretrained models along optimal generation paths, ensuring robust performance across diverse conditional settings. Additional qualitative results are provided in the supplement.

\section{Ablation Study}

\paragraph{Optimal Transport Cost.}  
Rather than maximizing similarity by setting the cost as $\mbf{M}_t = 1 - \mbf{Q}_t \mbf{K}_t^\top$, ASAG defines the cost as $\mbf{M}_t = \mbf{Q}_t \mbf{K}_t^\top$, which explicitly minimizes similarity and disrupts semantic alignment. As shown in Tab.~\ref{tab_cost}, the similarity-maximizing variant also improves over vanilla sampling, likely due to extrapolation effects similar to CFG, but incurs higher inference time due to the increased complexity of the Sinkhorn optimization.

We further evaluate the extreme case using the uniform plan $\frac{1}{n^2}\mathbf{1}\mathbf{1}^\top$, which represents a theoretical upper bound on semantic disruption. Interestingly, while it shows performance improvement over the baseline, we observe a reduction in sample diversity. In contrast, our Sinkhorn-based formulation achieves a similar degree of semantic disruption with better diversity and stability, validating its effectiveness as a practical and principled guidance strategy.

\paragraph{Computational Complexity.}  
We further assess the efficiency of ASAG by measuring inference time and memory usage, as shown in Tab.~\ref{tab_infer}. As described in Algorithm~\ref{algo-ASA}, the Sinkhorn process includes early stopping based on a convergence threshold. In practice, only 2 iterations are sufficient in most cases to produce stable transport plans. Remarkably, using just 2 iterations increases inference time by only $+0.35$ seconds while even improving generation quality, confirming that ASAG is not only efficient but also effective under minimal Sinkhorn updates.

\begin{table}[t!]
\caption{Ablation study on transport cost for Sinkhorn. }
\vspace{-0.5em}
\label{tab_cost}
\centering
\resizebox{0.95\linewidth}{!}{
\begin{small}
\begin{tabular}{lcccc}
\toprule
Method & FID & KID & Inception Score & Sinkhorn Iteration \\
\cmidrule(r){1-1}\cmidrule(r){2-5}\
Vanilla & 122.07 & 0.086 & 7.052 & - \\
$\mbf{M}_t$ = $1-\mbf{Q}\mbf{K}^{\top}$ & 111.53 & 0.078 & 9.085 &  $\approx$ 10 \\
$\mbf{P}_t^*$ = $\frac{1}{n}\mathbf{1}\mbf{1}^{\top}$ & {92.11} & \bf{0.058} & {9.710} & - \\
\cmidrule(r){1-1}\cmidrule(r){2-5}\
$\mbf{M}_t$ = $\mbf{Q}\mbf{K}^{\top}$ & \bf{92.01} & {0.059} & \bf{10.54} & $\approx$ 2 \\
\bottomrule	        
\end{tabular}
\end{small}
}
\end{table}

\begin{table}[t!]
\caption{Comparison computation cost with various approaches. Inference time is measured per prompt.}
\vspace{-0.5em}
\label{tab_infer}
\centering
\resizebox{0.95\linewidth}{!}{
\begin{small}
\begin{tabular}{ccccc}
\toprule
 Method & CFG & PAG & SEG & Ours  \\
\cmidrule(r){1-1}\cmidrule(r){2-4}\cmidrule(r){5-5}
Inference Time (sec) $\downarrow$ & 1.198 & 1.280 & 1.513 & 1.551 \scriptsize{\textcolor{red}{(+0.35)}} \\
Memory (G) $\downarrow$ & 16.41 & 16.46 & 16.61 &  16.61 \scriptsize{\textcolor{red}{(+0.20)}}\\
\bottomrule	        
\end{tabular}
\end{small}
}
\end{table}

\begin{figure}[t!]
\centering
\includegraphics[width=0.9\linewidth]{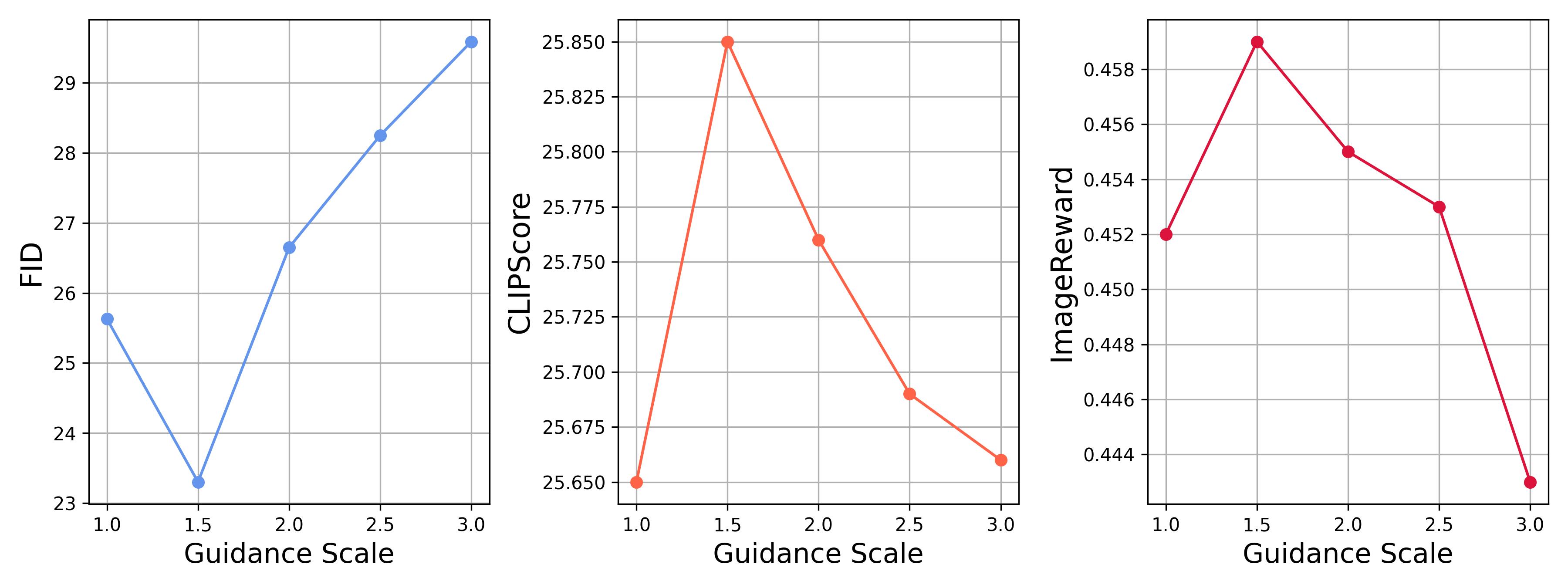}
\vspace{-1em}
\caption{Analysis of guidance scale across various metrics using our method.}
\label{fig:guidance_scale}
\vspace{-1em}
\end{figure}

\paragraph{Guidance Scale}  
To determine the optimal guidance scale $s$ in ASAG, we conduct experiments by varying $s$ and measuring the impact on generation performance, as shown in Figure~\ref{fig:guidance_scale}. We observe that performance improves with increasing $s$ up to a certain point, but overly large values cause slightly degradation as observed in CFG cases. The peak performance is achieved at $s = 1.5$, which we adopt as the default configuration throughout our experiments.

\section{Conclusion}

In this work, we propose \textbf{Adversarial Sinkhorn Attention Guidance (ASAG)}, a novel guidance sampling approach designed to address the lack of theoretical foundation behind existing perturbed self-attention methods for generating undesirable paths. Rather than relying on heuristic perturbations, we define an \emph{attention cost} that reflects interactions between pixel embeddings, and disrupt those interactions via the Sinkhorn algorithm—framing the process through the lens of optimal transport.

Our method is theoretically grounded and demonstrates state-of-the-art performance across various real-world scenarios, including both unconditional and conditional generation. Furthermore, we show that ASAG is highly compatible with downstream frameworks such as ControlNet and IP-Adapter, achieving significant improvements over existing guidance methods without requiring any additional training.

We believe ASAG not only provides a principled and effective guidance mechanism, but also offers both theoretical insight and practical utility within the diffusion generation framework. Our results suggest that ASAG can generalize to a wide range of generative tasks and pave the way for future research in attention-based guidance strategies.

\label{sec:reference_examples}

\clearpage
\appendix
\onecolumn
\setcounter{thm}{0}
\setcounter{lemma}{0}

\begin{center}
\Large \textbf{{Toward the Frontiers of Reliable Diffusion Sampling
\\via Adversarial Sinkhorn Attention Guidance}}\\
     {Supplementary Material}
\end{center}

In this supplementary document, we present the following:
\begin{itemize} 
\item Detailed description of the evaluation metrics and implementation. 
\item The proof of Theorem 1 and Lemma 1. 
\item Additional qualitative results, including unconditional and conditional case with SDXL and SD3, further ablation studies.
\end{itemize}

\section{Metrics and Implementation Detail} \label{sec:metric_detail}
For the results in Table~\ref{tab_main}, we conduct unconditional image sampling using 30,000 samples generated with SDXL. For the conditional case, we randomly select text prompts from the MSCOCO validation set. To evaluate the robustness of guidance scales, we perform conditional generation using Classifier-Free Guidance (CFG) with scale values uniformly sampled from the range (3, 5). For both PAG and SEG, we fix the scale to 3.0, following the settings recommended in their respective papers.

For Table~\ref{tab_hps}, we sample 200 prompts from DrawBench~\cite{drawbench} and 400 prompts from HPD~\cite{hpsv2}, generating 5 images per prompt. Additionally, for the ablation study in Figure~\ref{fig:guidance_scale}, we generate 5,000 images from the MSCOCO validation set using CFG-based sampling.

\section{Proof of Theorem 1}
\begin{thm}[Entropy-Maximizing Plan via Adversarial Sinkhorn]
Let $\mbf{Q}_t, \mbf{K}_t \in \mathbb{R}^{n \times d}$ be the query and key matrices at diffusion timestep $t$. Define the adversarial cost matrix as $\mbf{M}_t^{\downarrow} = (\mbf{Q}_t\mbf{K}_t^{\top})$. The entropy-regularized OT problem is defined as 
\(d^{\lambda}_{\mbf{M}_t^{\downarrow}}(\bs{\bs{\mu}},\bs{\nu}) = \underset{\bs{P}\in \mc{U}(\bs{\bs{\mu}},\bs{\nu})}{\min} \langle\mbf{P},\mbf{M}_t^{\downarrow}\rangle -\frac{1}{\lambda}\langle\mbf{P},\log \mbf{P}\rangle.
\)
Then in the limit $\lambda \to 0$ (i.e., $1/\lambda \to \infty$), the solution converges to the maximum-entropy plan:
\[
\lim_{\lambda \to 0} \mbf{P}_t^* = \frac{1}{n^2} \mathbf{1} \mathbf{1}^\top.
\]
\begin{proof}
The entropy‐regularized OT problem admits the closed‐form Sinkhorn solution~\cite{sinkhorn}:
\[
\mbf{P}^* \;=\;\mathrm{diag}(\bs{u})\;K\;\mathrm{diag}(\bs{v}),
\quad
K_{ij} = \exp\bigl(-\lambda\,\mbf{M}_{t,ij}^\downarrow\bigr),
\]
with scaling vectors \(\bs{u},\bs{v}>0\) chosen so that \(\mbf{P}^*\mathbf1=\bs{\bs{\mu}}\) and \((\mbf{P}^*)^\top\mathbf1=\bs{\nu}\).

As \(\lambda\to0\), we have \(K_{ij}=\exp(-\lambda\,\mbf{M}_{t,ij}^\downarrow)\to1\), hence
\[
P^*
=\mathrm{diag}(\bs{u})\,K\,\mathrm{diag}(\bs{v})
\;\longrightarrow\;
\mathrm{diag}(\bs{u})\,\mathbf1\mathbf1^\top\,\mathrm{diag}(\bs{v})
=\;\bs{u}\,\bs{v}^\top.
\]
Under uniform marginals \(\bs{\mu}=\bs{\nu}=\tfrac1n\mathbf1\), the constraints
\[
\mbf{P}^*\mathbf1 = \bs{u}\,(\bs{v}^\top\mathbf1) = \tfrac1n\mathbf1,
\qquad
(P^*)^\top\mathbf1 = \bs{v}\,(\bs{u}^\top\mathbf1) = \tfrac1n\mathbf1
\]
force \(\bs{u}=\tfrac1n\mathbf1\) and \(\bs{v}=\tfrac1n\mathbf1\). 

Therefore,
\[
\lim_{\lambda\to0}\mbf{P}^*
= \bs{u}\,\bs{v}^\top
= \frac1n\mathbf1\;\bigl(\frac1n\mathbf1\bigr)^\top
= \frac1{n^2}\,\mathbf1\,\mathbf1^\top,
\]
as claimed.
\end{proof}
\end{thm}

\section{Proof of Lemma 1}
\begin{lemma}[Uniform Plan Maximizes Entropy]
Under uniform marginals \(\bs{\mu} = \tfrac{1}{n}\mathbf{1}\) and \(\bs{\nu} = \tfrac{1}{n}\mathbf{1}\), the coupling
\[
\mbf{P}^* = \frac{1}{n^2}\,\mathbf{1}\,\mathbf{1}^\top
\]
uniquely maximizes the Shannon entropy
\[
H(P) \;=\; -\sum_{i,j} P_{ij}\,\log P_{ij}
\]
over the transport polytope \(\mathcal U(\bs{\mu},\bs{\nu})\).
\end{lemma}

\begin{proof}
Introduce Lagrange multipliers \(\{\alpha_i\}_{i=1}^n\) and \(\{\beta_j\}_{j=1}^n\) for the row and column constraints, respectively:
\[
\sum_j P_{ij} = \mu_i = \tfrac{1}{n},
\quad
\sum_i P_{ij} = \nu_j = \tfrac{1}{n}.
\]
The Lagrangian is
\[
\mathcal{L}(P,\alpha,\beta)
= -\sum_{i,j} P_{ij}\log P_{ij}
+ \sum_i \alpha_i\Bigl(\sum_j P_{ij} - \tfrac{1}{n}\Bigr)
+ \sum_j \beta_j\Bigl(\sum_i P_{ij} - \tfrac{1}{n}\Bigr).
\]
Stationarity \(\partial \mathcal{L}/\partial P_{ij}=0\) yields
\[
-\bigl(\log P_{ij} + 1\bigr) + \alpha_i + \beta_j = 0
\quad\Longrightarrow\quad
P_{ij} = \exp(\alpha_i + \beta_j - 1) = u_i\,v_j,
\]
where we set \(u_i = e^{\alpha_i-1}\) and \(v_j = e^{\beta_j}\).

Enforcing the marginals gives
\[
\sum_j u_i v_j = u_i\sum_j v_j = \tfrac{1}{n},
\quad
\sum_i u_i v_j = v_j\sum_i u_i = \tfrac{1}{n}.
\]
Thus \(\sum_j v_j = \sum_i u_i = 1\) and
\(\;u_i = v_j = \tfrac{1}{n}\) for all \(i,j\).  Hence
\(\;P_{ij} = u_i v_j = \tfrac{1}{n^2}\).

Finally, since \(H(P)\) is strictly concave on \(\mathcal U(\bs{\mu},\bs{\nu})\), this stationary point is the unique global maximizer.
\end{proof}

\section{Additional Qualitative Results} \label{sec:add_example}

In this section, we present additional qualitative results to further demonstrate the effectiveness and versatility of our proposed method, ASAG, across various generation tasks and in combination with other guidance approaches.

\paragraph{Comparison of Guidance Sampling with Our Method}
Fig.\ref{fig:uncond_supple} and \ref{fig:cond_supple} provide qualitative comparisons between ASAG and existing guidance methods including Vanilla, CFG, PAG, and SEG. In both unconditional and conditional generation settings, while existing methods often enhance visual appearance, they tend to distort the original structural and semantic consistency. In contrast, ASAG not only significantly improves visual quality but also faithfully preserves the structural layout and semantic intent of the original input.

\paragraph{Effect of Guidance Scale in ASAG}
Fig.~\ref{fig:uncond_scale}, \ref{fig:canny_supple}, and \ref{fig:depth_supple} illustrate the impact of varying the guidance scale in ASAG, under both unconditional generation and conditional generation using ControlNet. We observe that increasing the ASAG guidance scale consistently enhances overall generation quality across different input conditions, validating the effectiveness of our approach. In particular, for unconditional generation, guidance scales of 1.5 and 2.0 strike a good balance between quality and stability. For ControlNet-based conditional tasks, scales in the range of 1.5 to 3.0 are effective in achieving high visual fidelity.

\paragraph{Comparison with and without CFG}
Fig.~\ref{fig:pose_guidance} shows qualitative results under the pose condition with and without classifier-free guidance (CFG) using ControlNet. Interestingly, PAG fails to deliver meaningful improvements without CFG. In contrast, ASAG surprisingly achieves high-quality generations even without CFG. Furthermore, when CFG is applied, our method further boosts generation quality, indicating that ASAG serves as a strong and reliable guidance sampling strategy in both CFG and non-CFG scenarios.

\begin{figure*}[t!]
\centering
\includegraphics[width=0.9\linewidth]{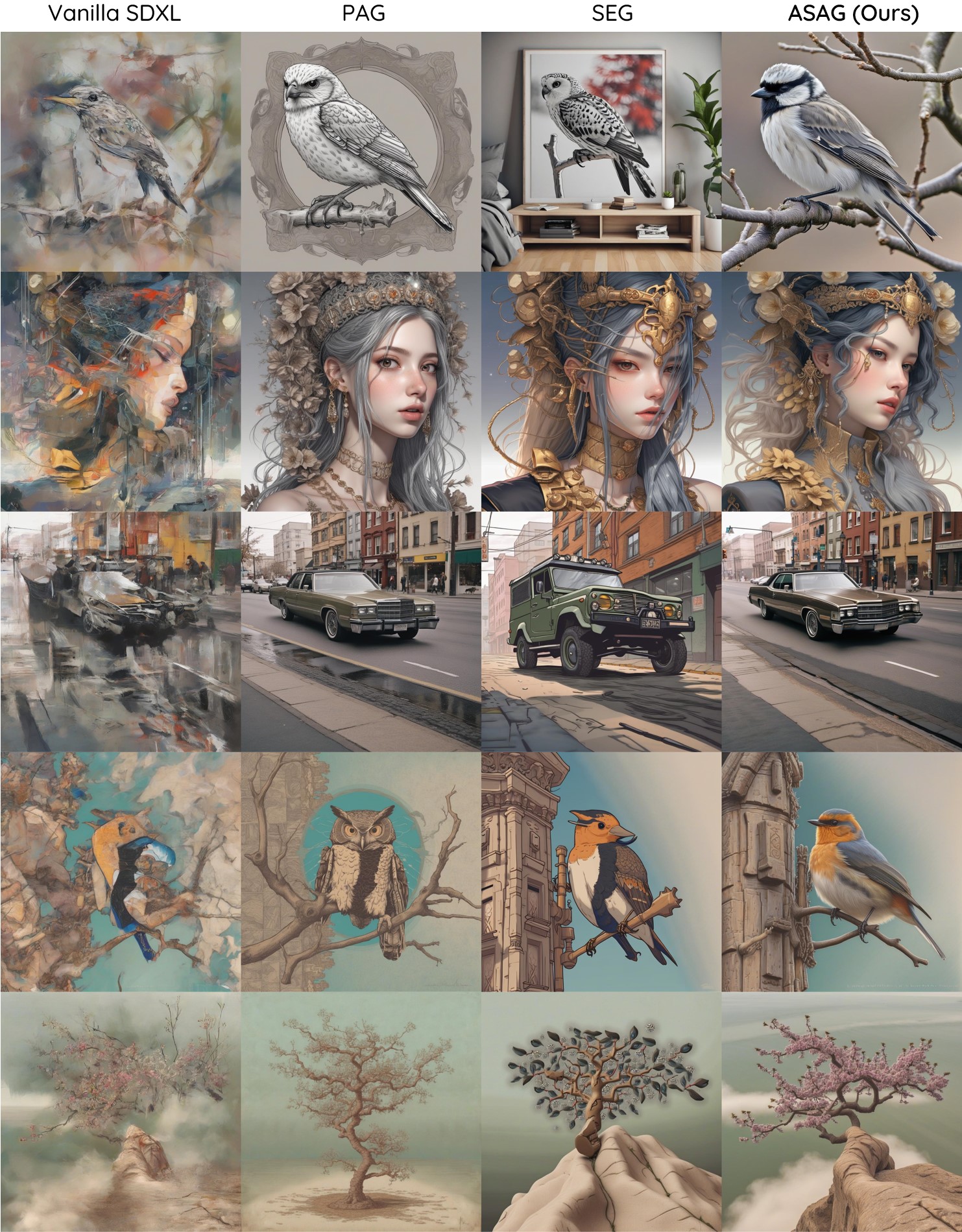}
\vspace{-0.5em}
\caption{Qualitative comparison on unconditional generation.
Unlike other guidance methods that often distort the original structure, ASAG maintains structural and semantic consistency while significantly improving visual fidelity. }
\label{fig:uncond_supple}
\vspace{-0.5em}
\end{figure*}

\begin{figure*}[t!]
\centering
\includegraphics[width=0.9\linewidth]{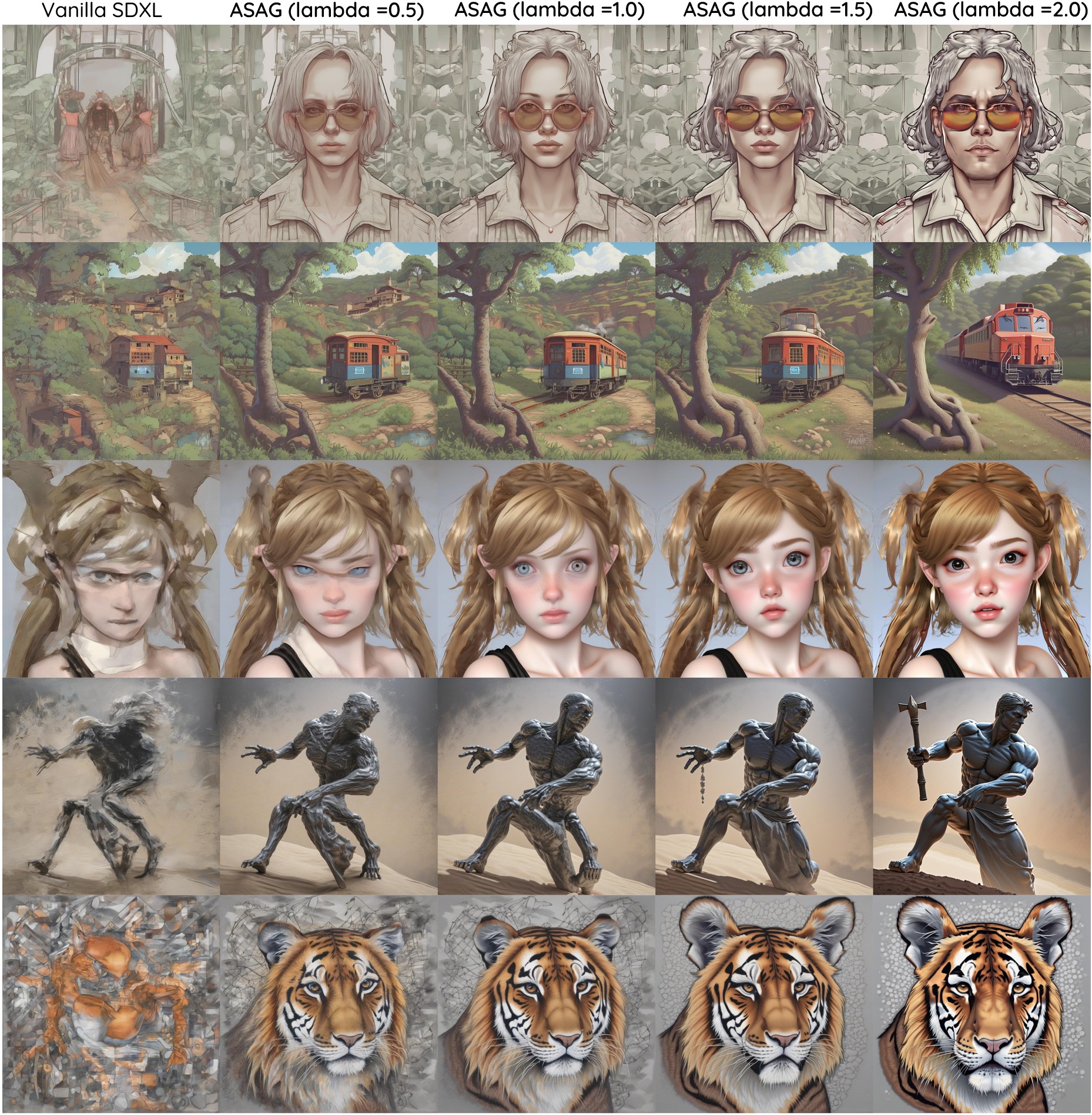}
\vspace{-0.5em}
\caption{Unconditional generation results with varying ASAG guidance scales.
Increasing the ASAG scale consistently improves visual quality over vanilla SDXL. Scales of 1.5 and 2.0 are found to be effective configurations for balancing quality and stability.}
\label{fig:uncond_scale}
\vspace{-0.5em}
\end{figure*}

\begin{figure*}[t!]
\centering
\includegraphics[width=0.8\linewidth]{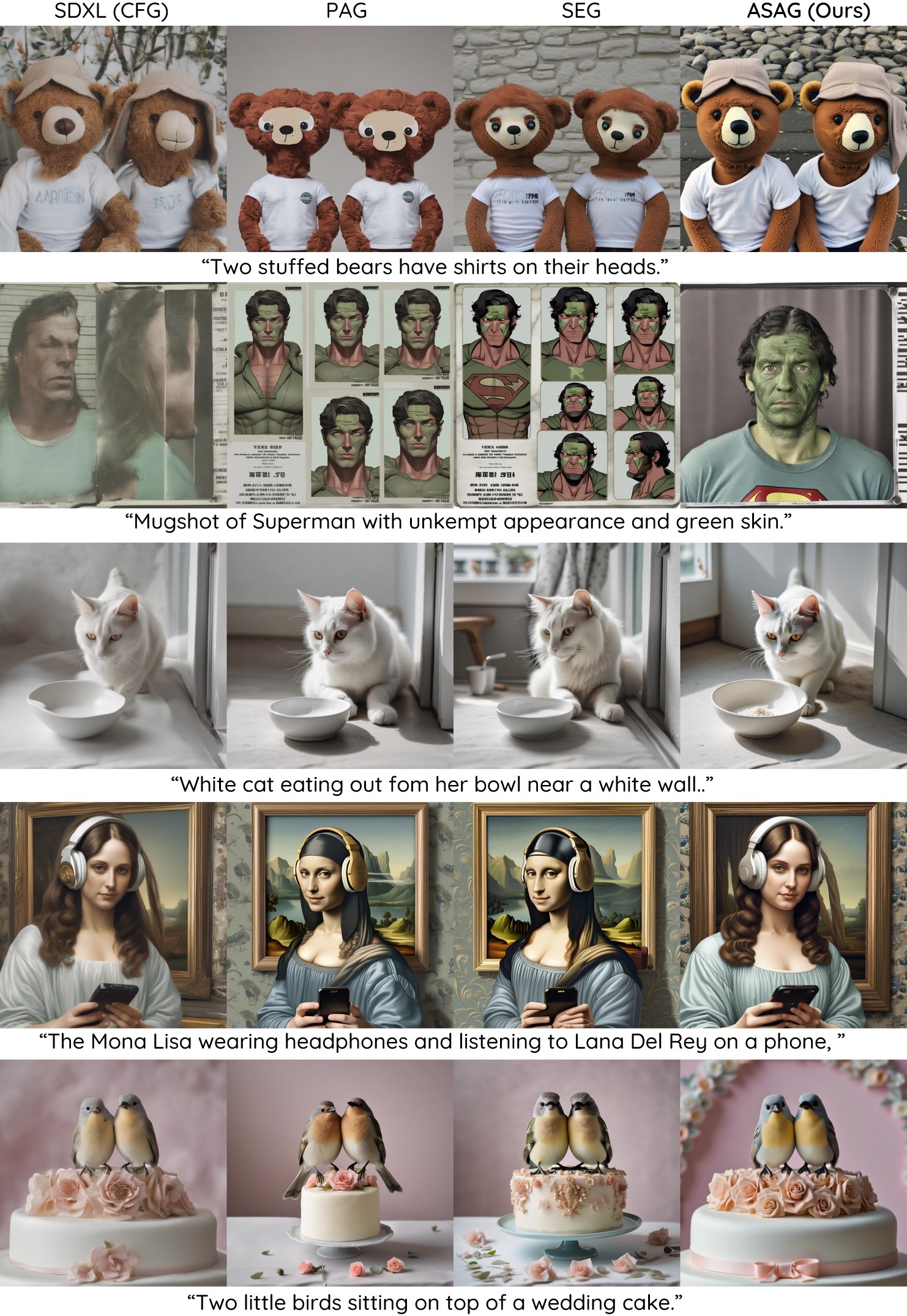}
\vspace{-0.5em}
\caption{Conditional generation results with CFG, PAG, SEG, and ASAG.
Unlike other methods that frequently disrupt structural integrity, ASAG enhances visual fidelity while retaining the original layout and semantic alignment.}
\label{fig:cond_supple}
\vspace{-0.5em}
\end{figure*}

\begin{figure*}[t!]
\centering
\includegraphics[width=1\linewidth]{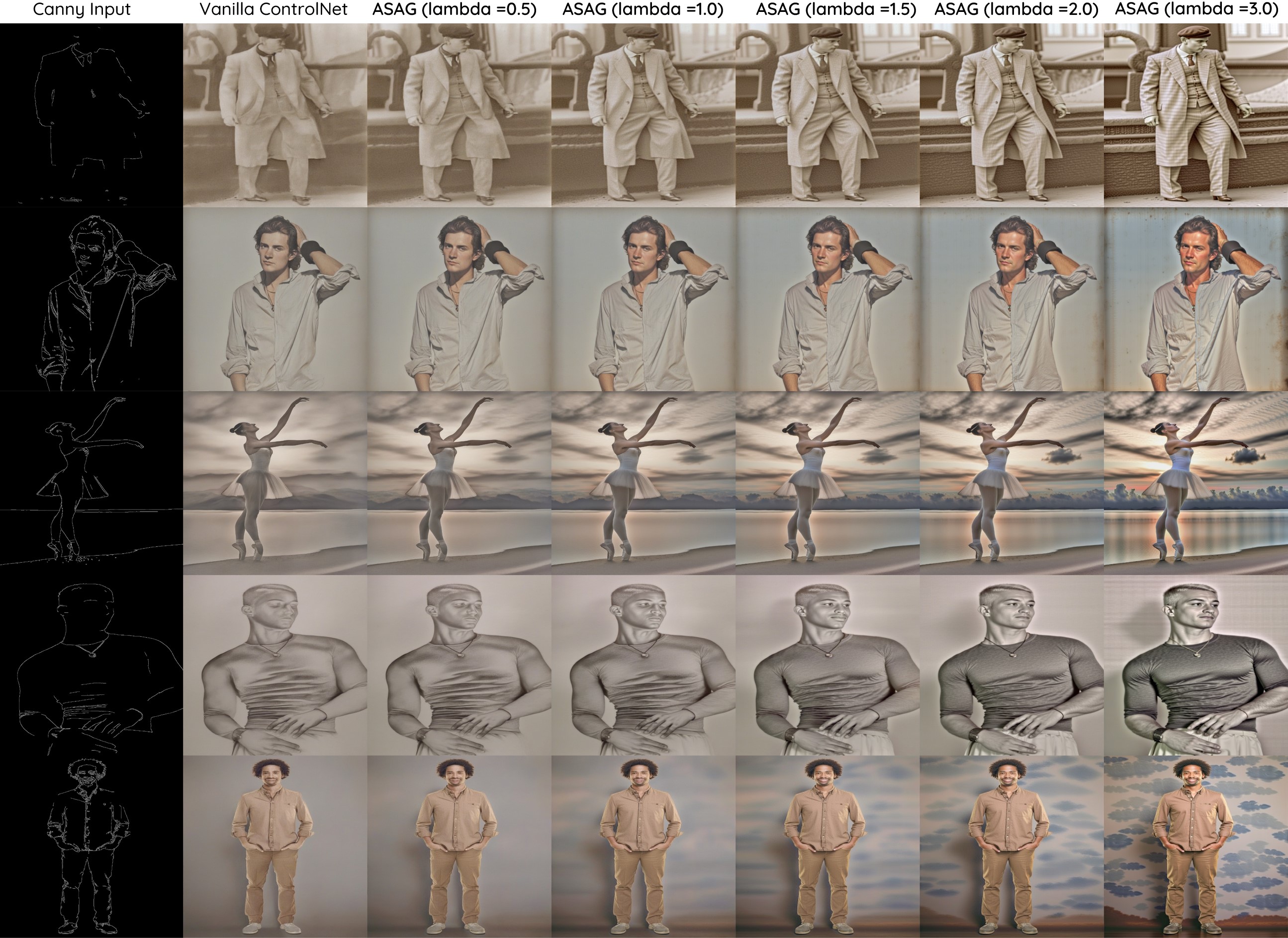}
\vspace{-0.5em}
\caption{Comparison results on Canny-conditioned generation with ControlNet using different guidance scales.
We observe that guidance scales in the range of 1.5 to 3.0 yield enhanced visual quality while maintaining stable generation, making them suitable configurations for ControlNet.}
\label{fig:canny_supple}
\vspace{-0.5em}
\end{figure*}

\begin{figure*}[t!]
\centering
\includegraphics[width=1\linewidth]{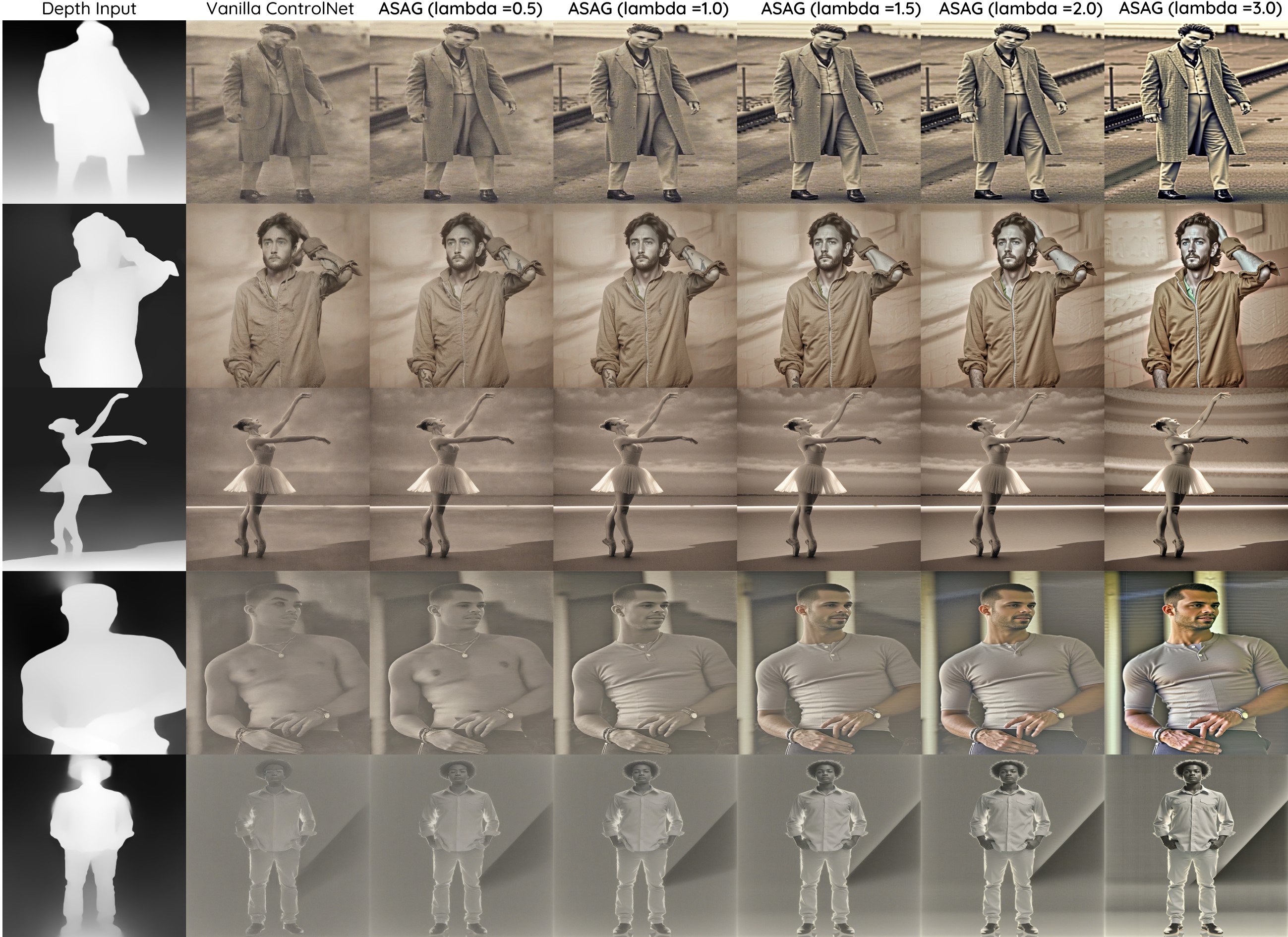}
\vspace{-0.5em}
\caption{Comparison results on depth-conditioned generation with ControlNet using different guidance scales.
We find that guidance scales in the range of 1.5 to 3.0 are effective for ControlNet, yielding enhanced visual quality while maintaining stable generation.
}
\label{fig:depth_supple}
\vspace{-0.5em}
\end{figure*}

\begin{figure*}[t!]
\centering
\includegraphics[width=0.9\linewidth]{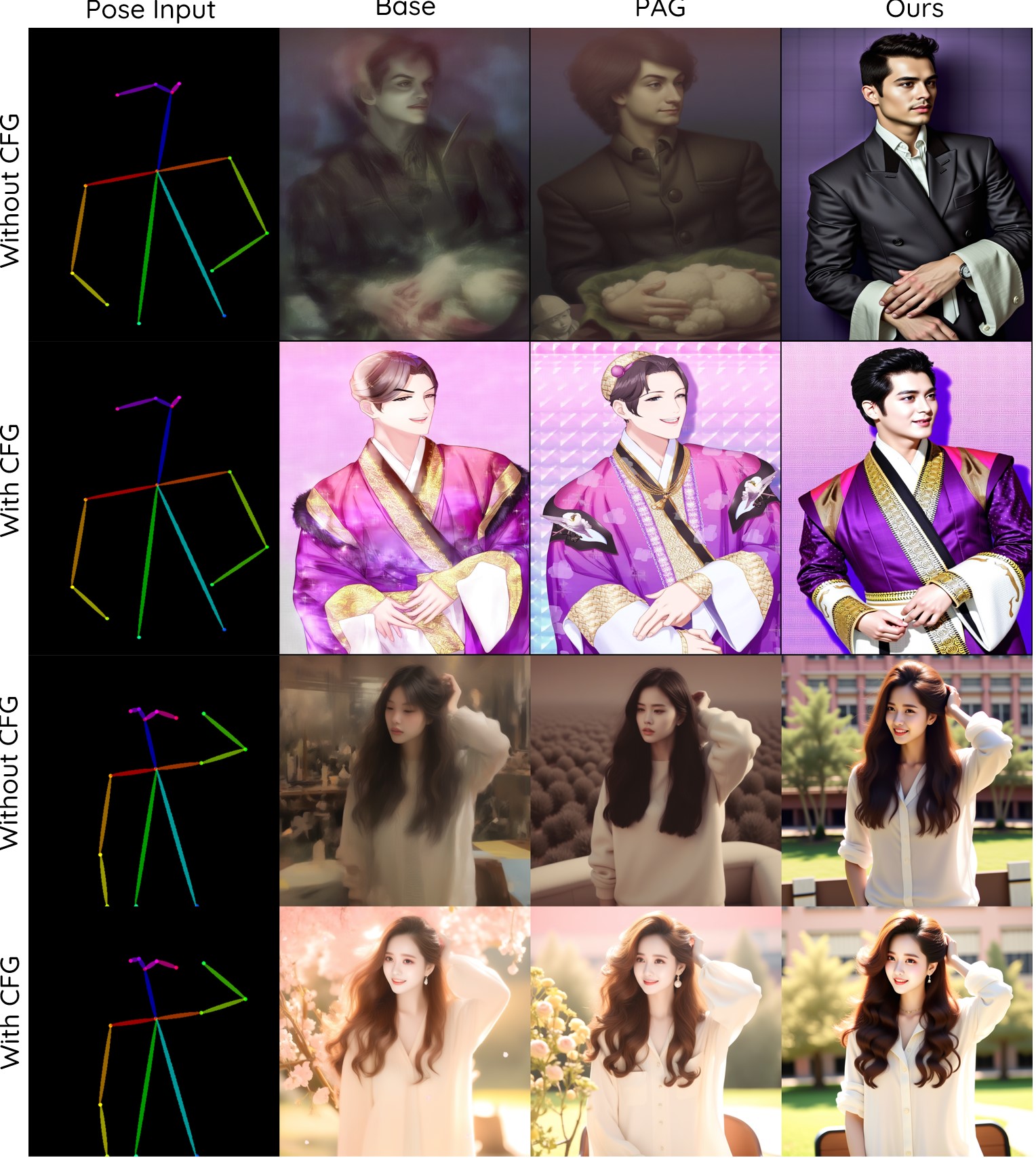}
\vspace{-0.5em}
\caption{Comparison results on ControlNet under pose condition, with and without CFG.
ASAG produces visually pleasing results even without CFG, whereas PAG fails to generate satisfactory outputs in the same setting. Furthermore, when combined with CFG, our method further boosts generation quality. }
\label{fig:pose_guidance}
\vspace{-0.5em}
\end{figure*}

\end{document}

%% file: preamble.tex
\usepackage{times}
\usepackage{epsfig}
\usepackage{graphicx}
\usepackage{amsmath}
\usepackage{amssymb}

\usepackage{url}
\usepackage{graphicx}
\usepackage{booktabs}
\usepackage[table,dvipsnames]{xcolor}
\usepackage[table,x11names]{xcolor}
\usepackage{amsmath}
\usepackage{amssymb}
\usepackage{amsthm}
\usepackage{pifont}
\usepackage{lipsum}
\usepackage[export]{adjustbox}
\usepackage{capt-of}
\usepackage[linesnumbered,ruled,vlined]{algorithm2e}
\usepackage{kotex}
\usepackage{thm-restate}
\usepackage{multirow}
\usepackage{caption}
\usepackage[figoff]{figcaps}

\usepackage{multicol}

\newcommand{\mbx}{\mathbf{x}}

\newcommand{\mbc}{\mathbf{c}}
\newcommand{\me}{\boldsymbol{\epsilon}}

\newcommand{\mc}{\mathcal}
\newcommand{\mbf}{\mathbf}
\newcommand{\bs}{\boldsymbol}

\usepackage{booktabs}  
\newtheorem{thm}{Theorem}
\newtheorem{remark}{Remark}

\newtheorem{corollary}{Corollary}[thm]
\newtheorem{lemma}{Lemma}
\usepackage{fontawesome5}
\usepackage{resizegather}

